\def\year{2020}\relax
\documentclass[letterpaper]{article} 
\usepackage{aaai20}  
\usepackage{times}  
\usepackage{helvet} 
\usepackage{courier}  
\usepackage[hyphens]{url}  
\usepackage{graphicx} 
\usepackage{graphicx}  
\usepackage{amsthm}         
\usepackage{amssymb}
\theoremstyle{definition}        
\usepackage{mathtools}   
\usepackage{bm}      
\usepackage{amsfonts}  
\newcommand{\norm}[1]{\left\lVert#1\right\rVert}    
\usepackage{tikz}     

\newtheorem{prop}{Proposition}    
\usepackage{multirow} 
\usepackage{caption} 
\title{Hierarchical Modes Exploring in Generative Adversarial Networks}
\author{Mengxiao Hu, \Large \textbf{Jinlong Li}, \Large \textbf{Maolin Hu, Tao Hu}\\ 
University of Science and Technology of China\\ 
m\_x\_hu@126.com, jlli@ustc.edu.cn, \{humaolin,Skyful\}@mail.ustc.edu.cn 
}
 
\begin{document}

\maketitle

\begin{abstract}
In conditional Generative Adversarial Networks (cGANs), when two different initial noises are concatenated with the same conditional information, the distance between their outputs is relatively smaller, which makes minor modes likely to collapse into large modes. To prevent this happen, we proposed a hierarchical mode exploring method to alleviate mode collapse in cGANs by introducing a diversity measurement into the objective function as the regularization term. We also introduced the Expected Ratios of Expansion (ERE) into the regularization term, by minimizing the sum of differences between the real change of distance and ERE, we can control the diversity of generated images w.r.t specific-level features. We validated the proposed algorithm on four conditional image synthesis tasks including categorical generation, paired and un-paired image translation and text-to-image generation. Both qualitative and quantitative results show that the proposed method is effective in alleviating the mode collapse problem in cGANs, and can control the diversity of output images w.r.t specific-level features. 
\end{abstract}

\section{Introduction}
    With the potentiality of capturing high dimensional probability distributions, Generative Adversarial Networks (GANs) \cite{173} are broadly used in synthesizing text \cite{174}, videos \cite{41} and images \cite{10}.  Conditional GANs (cGANS) \cite{178} are one of the early variants of GANs and have been applied in many tasks of image synthesis \cite{160} because of the ability to synthesizing images with given information (e.g, generating images of bird with given description of colors).
 
   Many generation tasks adopt GANs for its simple setting and impressive result, but we often suffer from the problem that the generator can only synthesize samples from few modes of  the real data distribution, which is called "mode collapse". The formal definition of mode collapse \cite{3} provides a theoretical measure of mode collapse.

In image synthesis, mode collapse means the output images are less diverse than the real ones. Therefore, it might be very important to quantify the diversity of output images for addressing the mode collapse problem. The Learned Perceptual Image Path Similarity (LPIPS) \cite{186} and the Fr\'{e}chet Inception Distance (FID) \cite{187} are often used to measure the diversity and the quality of output images,  respectively. The Number of Statistically-Different Bins (NDB) \cite{5} is also used for estimation of mode missing in the generated distribution \cite{115}. 
Unlike previous works only use the diversity metrics for evaluation, in this work, we proposed hierarchical Modes Exploring  Generative Adversarial Networks to alleviate the mode collapse problem in cGANs by introducing a diversity measurement into the objective function as a regularization term.

    The regularization term was employed to expand the difference of an output image pair which is obtained by feeding an input pair from the same batch to the generator \cite{115,175}. Here, we firstly compute the ratio of the distance between an input pair to the distance between an output feature pair and use it as a coefficient of the expansion at each convolutional layer of the generator. Then, we calculate the absolute difference between the computed ratio and a predefined ratio which is used to control the generated images with different features. At last, we sum the absolute differences across all layers as the regularization term of the objective function. Since  our regularization method requires no modification to the structure of original networks, it can be used in cGANs for various tasks.

   In this work, our primary contributions are: 
\begin{quote}
\begin{itemize}
    \item We proposeed a hierarchical mode exploring method to alleviate mode collapse in cGANs by introducing a diversity measurement into the objective function as the regularization term. 
    \item We introduced the Expected Ratio of Expansion (ERE) into the regularization term. With different ERE, we can control the diversity of generated images w.r.t specific-level features. 
    \item We demonstrated the proposed regularization method on different datasets in three image synthesis tasks, and experimental results show that our method can generate images with higher  diversity, compared with the baseline models.

\end{itemize}
\end{quote}

\section{Backgrounds}
    Generative adversarial networks (GANs) are composed of two players: the generator \textit{G} and the discriminator \textit{D}. The training process of GANs is a minimax game between \textit{D} and  \textit{G} . \textit{G} is learning to transform the initial noise to data with same dimension as real data’s, such that \textit{D} cannot tell whether the output was drawn from the true distribution or was generated by \textit{G}. The solution of this minimax game is a Nash equilibrium in which neither \textit{G} nor \textit{D} can improve unilaterally.
    
    In cGANs, because with the same concatenated information vectors, the distance between two inputs is smaller than the distance between two noise vectors, and the shrinkage of distance is very likely to be persevered through a upsampling layer (e.g., a fractionally-strided convolution layer) which is necessary for \textit{G}. So, we can expect the distance between two outputs of \textit{G} is smaller,compared to the one in standard GANs. To offset the shrinkage of distance, people used a regularization term $\mathcal{L}_{d}$ to maximize the distance between two output images \cite{115,175}.
    \begin{equation}\label{eq:1}
\begin{split}
    \mathcal{L}_{d} &= \norm{\frac{d^{(1)}(\mathbf{z_1},\mathbf{z_2})}{d^{(n)}(G^{(n)}(\mathbf{z_1}),G^{(n)}(\mathbf{z_2}))}} _{1}
\end{split}
\end{equation}
Where, $\mathbf{z}$ is the latent code, $d^{(i)}(\cdot)$ refers the distance metric, $G^{(i)}(\mathbf{z})$ is the output of $i$-th convolutional layer in the generator when the input is $\mathbf{z}$, and $n$ is the number of convolutional layers in the generator.

\section{Methods}
\subsection{Regularization term}
 Given an input pair $(\mathbf{z_1},\mathbf{z_2})$, we firstly measure the distance between the output pair of every convolutional layer, 
then, we compute the ratio of the distance at one layer to the distance at the next layer: 
\begin{equation}\label{eq:2}
ratio^{(i)} = \frac{d^{(i-1)}(G^{(i-1)}(\mathbf{z_1}),G^{(i-1)}(\mathbf{z_2}))}{d^{(i)}(G^{(i)}(\mathbf{z_1}),G^{(i)}(\mathbf{z_2}))}
\end{equation}
As shown by Eq. \eqref{eq:2}, we denote the sum of $L_1$ norms of the difference between the computed ratio and a target ratio as the regularization term:
\begin{equation}\label{eq:3}
\mathcal{L}_h = \sum_{i=2}^{n}\norm{ratio^{(i)} - \lambda^{(i)}} _1
\end{equation}
  Where, $\lambda^{(i)}$ is the hyper-parameter to control the diversity gain though $i$-th layer. When $\lambda^{(i)}=0$ for $\forall i \in \{2,...,n \}$, we maximize the diversity of the output images by minimizing the regularization term $\mathcal{L}_h$.

The proposed regularization method offsets the shrinkage of distance between two outputs at every layer in the generator. When we minimize $\mathcal{L}_h$ with $\lambda^{(i)}$ being set to an appropriate value, it can alleviate the mode collapse problem at every layer of the generator.

To illustrate the advantage of $\mathcal{L}_h$, we compared $\mathcal{L}_h$ with $\mathcal{L}_d$ proposed by Mao \textit{et al.}, and there is a true proposition.

\begin{prop}
Denotes the rate of $ratio_{(i)}$ converging to $0$ as $r^{(i)}$. If $\exists i\in\{2,...,n\}$ such that $r^{(i)} \gg r^{(i')}$ for $\forall i'\in\{2,...,n\}\setminus\{i\}$, in both training processes using $\mathcal{L}_{h}$ and $\mathcal{L}_{d}$, namely, $r^{(i)}_{l} \gg r^{(i')}_{l'}$, for $\forall l,l'\in\{\mathcal{L}_{h},\mathcal{L}_{d}\}$, then, $\exists H\in(0,+\infty), (\mathcal{L}_{h}\leq H, and, \mathcal{L}_{d}\leq H) \implies l(t_{h}) < l(t_{d})$. Here, when $l \leq H$, the corresponding training process is immediately stopped, $t_{h}$ is the stop time when the regularization term is $\mathcal{L}_{h}$.
\end{prop}
\begin{proof}
Assume $ratio^{(i)}_{\mathcal{L}_{h}}(t)=ratio^{(i)}_{\mathcal{L}_{d}}(t)=e^{-r^{(i)}_{\mathcal{L}_{h}}t}$. According to the condition in the proposition, assume  $ratio^{(i')}_{\mathcal{L}_{h}}(t) \approx ratio^{(i')}_{\mathcal{L}_{d}}(t) = e^{-r^{(i')}_{\mathcal{L}_{h}}t}$, and $r^{(i)}_{\mathcal{L}_{h}} \gg r^{(i')}_{\mathcal{L}_{h}}$, such that $e^{-r^{(i')}_{\mathcal{L}_{h}}\Bar{t}}>2e^{-r^{(i)}_{\mathcal{L}_{h}}\Bar{t}}$ for a big $\Bar{t}$, such that $e^{-r^{(i)}_{\mathcal{L}_{h}}\Bar{t}}\approx 0$.

Denotes $a=e^{-r^{(i')}_{\mathcal{L}_{h}}\Bar{t}}-e^{-r^{(i)}_{\mathcal{L}_{h}}\Bar{t}}$, so, $a>e^{-r^{(i)}_{\mathcal{L}_{h}}\Bar{t}}>0$, $a\in(0,1)$, and, $e^{-r^{(i')}_{\mathcal{L}_{h}}\Bar{t}}=a+e^{-r^{(i)}_{\mathcal{L}_{h}}\Bar{t}}$. 

It is convenient for the proof to assume $ratio^{(i)}_{\mathcal{L}_{h}}(t)=ratio^{(i)}_{\mathcal{L}_{d}}(t)$ , we will discuss the other cases later.

According to Eq. \eqref{eq:1},
\begin{equation}\label{eq:4}
\begin{split}
    \mathcal{L}_{d}(\Bar{t}) &= \norm{\frac{d^{(1)}(\mathbf{z_1},\mathbf{z_2})}{d^{(n)}(G^{(n)}(\mathbf{z_1}),G^{(n)}(\mathbf{z_2}))}} _{1} = \prod_{j=2}^{n} \norm{ratio^{(j)}_{\mathcal{L}_{d}}(\,\Bar{t}\,)}_{1} \\ 
    &= (a+e^{-r^{(i)}_{\mathcal{L}_{h}}\Bar{t}})^{n-2}e^{-r^{(i)}_{\mathcal{L}_{h}}\Bar{t}}\approx 0 \approx e^{-r^{(i)}_{\mathcal{L}_{h}}\Bar{t}}
\end{split}
\end{equation}
According to Eq. \eqref{eq:2}, when $\lambda_{i} = 0$,
\begin{equation}\label{eq:5}
\begin{split}
    \mathcal{L}_{h}(\Bar{t}) &=\sum_{j=2}^{n} \norm{ratio^{(j)}_{\mathcal{L}_{h}}(t)(\,\Bar{t}\,)}_{1} \\
    &= (n-2)(a+e^{-r^{(i)}_{\mathcal{L}_{h}}\Bar{t}})+e^{-r^{(i)}_{\mathcal{L}_{h}}\Bar{t}} \\
    &> (n-2)a > 0
\end{split}
\end{equation}
Denotes $H=(n-2)a$, so, $H\in(0,+\infty)$, and, at time $\Bar{t}$, $\mathcal{L}_{d}(\Bar{t})<H$, the corresponding training process has been stopped, suppose it stopped at time $t_d$, $t_{d}\leq \Bar{t}$; and at time $\Bar{t}$, $\mathcal{L}_{h}(\Bar{t})>H$, the corresponding training process will be stopped at time $t_h$, $t_{h}> \Bar{t}>t_{d}$; therefore, according to the monotonicity of $l$, $l(t_h)<l(t_d)$.

If $ratio^{(i)}_{\mathcal{L}_{h}}(t)>ratio^{(i)}_{\mathcal{L}_{d}}(t)$, namely, $r^{(i)}_{\mathcal{L}_{h}}<r^{(i)}_{\mathcal{L}_{d}}$, denotes the new stop time as $t^{>}_{d}$ and the new regularization term as $\mathcal{L}^{>}_{d}$, so, when $\mathcal{L}_{d}({t_d})=\mathcal{L}^{>}_{d}({t^{>}_{d}})$, according to Eq. \eqref{eq:4}, $e^{-r^{(i)}_{\mathcal{L}_{h}}t_{d}}=e^{-r^{(i)}_{\mathcal{L}_{d}}t^{>}_{d}}$,  $t^{>}_{d}=(r^{(i)}_{\mathcal{L}_{h}}/r^{(i)}_{\mathcal{L}_{d}})t_{d}<t_{d}$, so, $t_{h}>t_{d}>t^{>}_{d}$; therefore, $l(t_h)<l(t^{>}_{d})$.

If $ratio^{(i)}_{\mathcal{L}_{h}}(t)<ratio^{(i)}_{\mathcal{L}_{d}}(t)$, according to the condition in the proposition, assume  $ratio^{(i')}_{\mathcal{L}_{h}}(t) \approx ratio^{(i')}_{\mathcal{L}_{d}}(t)$, namely, $r^{(i')}_{\mathcal{L}_{d}} \approx r^{(i')}_{\mathcal{L}_{h}}$, and $r^{(i)}_{\mathcal{L}_{d}} \gg r^{(i')}_{\mathcal{L}_{d}}$, such that $e^{-r^{(i')}_{\mathcal{L}_{d}}\Bar{t}^{<}}>2e^{-r^{(i)}_{\mathcal{L}_{d}}\Bar{t}^{<}}$ for a big $\Bar{t}^{<}$, such that $e^{-r^{(i)}_{\mathcal{L}_{d}}\Bar{t}^{<}}\approx 0$, and Eq. \eqref{eq:4} and Eq. \eqref{eq:5} are still satisfied when $t=\Bar{t}^{<}$, the only difference between the old equations and the new equations is that $a$ in the new Eq. \eqref{eq:4} and Eq. \eqref{eq:5} is not the same, for clarity, denotes $a$ in the new Eq. \eqref{eq:4} and Eq. \eqref{eq:5} as $a_{4}$ and $a_{5}$, $a_{4}=e^{-r^{(i')}_{\mathcal{L}_{d}}\Bar{t}^{<}}-e^{-r^{(i)}_{\mathcal{L}_{d}}\Bar{t}^{<}}$, $a_{5}=e^{-r^{(i')}_{\mathcal{L}_{h}}\Bar{t}^{<}}-e^{-r^{(i)}_{\mathcal{L}_{h}}\Bar{t}^{<}}$; because $e^{-r^{(i)}_{\mathcal{L}_{d}}\Bar{t}^{<}}>e^{-r^{(i)}_{\mathcal{L}_{h}}\Bar{t}^{<}}$ and $r^{(i')}_{\mathcal{L}_{d}} \approx r^{(i')}_{\mathcal{L}_{h}}$, $a_{5}>a_{4}$; similarly, we can choose $H^{<}=(n-2)a_{5}$ and $t^{<}_h, t^{<}_d$ as stop time, then, the same conclusion can be deduced when $ratio_{i}^{\mathcal{L}_{h}}(t)=ratio_{i}^{\mathcal{L}_{d}}(t)$, namely, $l(t^{<}_h)<l(t^{<}_d)$.

\end{proof}

Proposition 1 shows that, when there exists a layer whose $ratio^{(i)}$ converges much faster than other layers', $\mathcal{L}_{d}$ converges more quickly than $\mathcal{L}_{h}$, which means the training process supervised by $\mathcal{L}_{d}$ stops earlier than the one supervised by $\mathcal{L}_{h}$, even though their cost are the same in any form of $\mathcal{L}_{h}$ or $\mathcal{L}_{d}$.

Since we use $\lambda^{(i)}$ to independently control $ratio^{(i)}$ at each layer, we need not assign different weights to them, therefore, the computation cost of searching the only weight of $\mathcal{L}_{h}$ is $\mathcal{O}(n)$. However, if we use $\mathcal{L}_{d}$ as the term to adjust the change of distance through specific layer of \textit{G}, the computation cost for searching the weights of terms grows exponentially with $n$.

\subsection{Expected Ratio of Expansion}

In Eq. \eqref{eq:3}, $\lambda^{(i)}$ controls the diversity of the output. For example, when $\lambda^{(i)} = 1$,  the $i$-th convolutional layer is encouraged to not change the distance; when $\lambda^{(i)}=0$ for $\forall i \in \{2,...,n \}$, every layer of the generator is to maximize the distance between the output images. $\pmb{\lambda}$ is set as the target ratio of the expansion, we call it the Expected Ratio of Expansion (ERE) in this work.

In practice, it is important to determine the value of $\lambda^{(i)}$, since when $\lambda^{(i)}$ is larger than $1$, the diversity is encouraged not to be increased, and because the distance cannot be $+\infty$, we cannot increase the diversity by setting $\lambda^{(i)}$ lower than its lower bound. Therefore, we restrict $\lambda^{(i)} \in [b^{(i)},1]$.

To compute $b^{(i)}$, there are two steps, firstly, we pre-trained the cGANs using:
\begin{equation}\label{eq:6}
\mathcal{L}_{fin} = \mathcal{L}_{ori}+ \beta \mathcal{L}_{h}^{0}
\end{equation}
Where, $\beta$ is the weight to manipulate the importance of the regularization. $\mathcal{L}_{h}^{0}$ denotes the $\mathcal{L}_h$ with all $\lambda^{(i)}=0$, and $\mathcal{L}_{ori}$ is the objective function used in the cGANs framework into which we integrated the proposed method. Then, we fed the whole dataset $X$ into the generator to calculate the ratio matrix: 
\begin{equation}\label{eq:7}
    A^{(i)} = 
    \begin{bmatrix}
    d_{11} &  \dots  & d_{1m} \\
    \vdots &  \ddots & \vdots \\
    d_{m1} &  \dots  & d_{mm}
\end{bmatrix}
\end{equation}
Here, $m$ is the size of $X$, $d_{uv} =  \frac{d^{(i-1)}(G^{(i-1)}(\mathbf{z_u}),G^{(i-1)}(\mathbf{z_v}))}{d^{(i)}(G^{(i)}(\mathbf{z_u}),G^{(i)}(\mathbf{z_v}))}$. If we choose $L_{1}$ norm as $d^{(i)}(\cdot)$ for $\forall i$, then $A^{(i)}$ can be calculated by: 
\begin{equation}\label{eq:8}
\begin{split}
    (A^{(i)})^{u}_{\,\,\, v} = \frac{|o^{u}M_{v,p}-o^{v}M_{u,p}|}{|o^{u}N_{v,q}-o^{v}N_{u,q}|}
\end{split}
\end{equation}
Here, $M$ and $N$ are the $1\times m \times f^{(i)}$ matrices output by  $(i-1)$-th layer and $i$-th layer, respectively, $f^{(i)}$ is the  dimension of output by $G^{(i)}(\cdot)$, $o$ is $\Vec{1}_{m\times 1}$.

Then, $b^{(i)}$ is determined as the minimum element of $A^{(i)}$,
 $b^{(i)} = min(d_{11}, ..., d_{jk}, ..., d_{mm})$.
    Since it requires 2 loops to calculate all $d_{uv}$ to determine $A^{(i)}$, the time complexity of naively computing $b^{(i)}$ is $O(m^{2})$. Eq. \eqref{eq:8} provides a way to compute $A^{(i)}$ in the form of tensors to compute $b^{(i)}$ with complexity $O(m)$, because tensor operation of a batch can be executed on GPU in parallel.

\begin{figure*}[h]
\centering
\includegraphics[width=160mm]{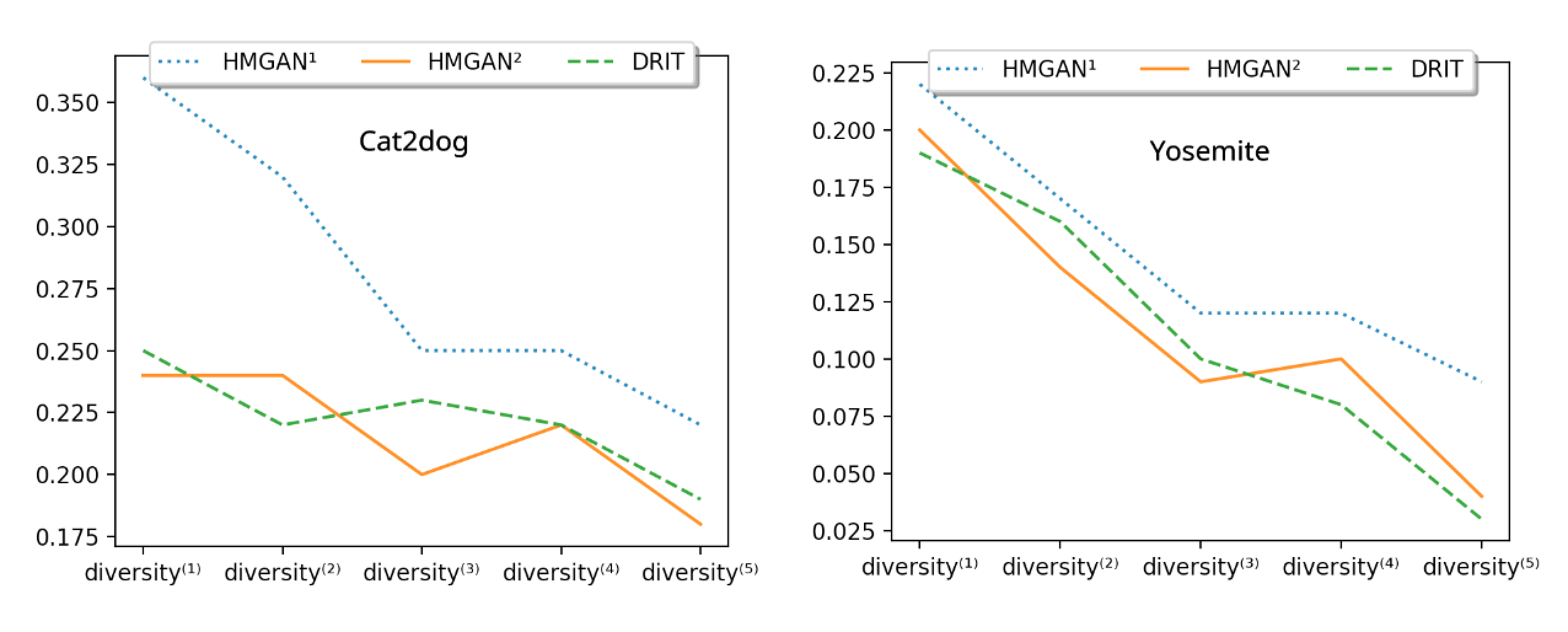}

\caption{Visualizing diversity of image batch w.r.t different-level features. HMGAN$^{1}$ refers $\lambda^{(i)}=0$ for $\forall{i}$, and HMGAN$^{2}$ refers $\lambda^{(i)}=1$. We measured the diversity with LPIPS which uses Alexnet as feature decoder, it provides outputs from 5 layers.} 
\label{figure:2}
\end{figure*}

\section{Experiments}
 To validate our regularization method under an extensive evaluation, we incorporated four baseline models (DCGAN \cite{194}, Pix2Pix \cite{191}, DRIT \cite{183} and StackGAN++ \cite{179}) with it for three conditional image synthesis tasks: 
\begin{itemize}
    \item Categorical generation, it is trained on CIFAR-10 \cite{192} using DCGAN as the baseline model.
    \item Image-to-image translation, it can be divided into two subtasks:
    \begin{itemize}
        \item Paired image-to-image translation, it is trained on facades  and maps using Pix2Pix as the baseline model.
    \item Unpaired image-to-image translation, it is trained on Yosemite \cite{195} and cat$\rightleftharpoons$dog \cite{183} using DRIT as the baseline model.
    \end{itemize}
    \item Text-to-image generation, it is trained on CUB-200-2011 \cite{196} using StackGAN++ as the baseline model.
\end{itemize}

Because the original networks of the baseline model do not change after adding the attention unit and the regularization term, we kept the hyper-parameters of the baseline model original. 

We adopted $L_1$ norm as distance metrics for all $d^{(i)}(\cdot)$ and set the weight of regularization $\beta = 1$ in all experiments.

\begin{table*}[ht]
\centering

\caption{NDB and JSD results on the CIFAR-10 dataset.}
\label{table:1}

\begin{tabular}{p{2cm}p{2cm}p{2cm}p{2cm}p{2cm}p{2cm}p{2cm}}
\hline
Metrics              & Models & airplane & automobile & bird  & cat   & deer  \\ \cline{3-7} 
\multirow{3}{*}{NDB $\downarrow$} & DCGAN  & $49.60\pm 3.50$    & $53.30\pm 6.34$      & $34.30\pm 5.71$ & $46.00\pm 2.65$ & $43.90\pm4.12$ \\
                     & HMGAN$^{1}$  & $\pmb{45.30\pm5.24}$    & $\pmb{51.50\pm 3.13}$      & $\pmb{33.20\pm 2.02}$ & $\pmb{42.00\pm 1.47}$ & $\pmb{42.20\pm 4.37}$ \\
                     & HMGAN$^{2}$  & $\pmb{48.70\pm 5.13}$    & $\pmb{52.90\pm2.98}$      & $\pmb{34.00\pm 2.38}$ & $\pmb{45.50\pm 2.12}$ & $\pmb{43.50\pm 4.14}$ \\ \hline
\multirow{3}{*}{JS $\downarrow$}  & DCGAN  & $0.035\pm0.002$    & $0.035\pm0.002$      & $0.026\pm 0.002$ & $0.031\pm 0.001$ & $0.033\pm 0.002$ \\
                     & HMGAN$^{1}$  & $\pmb{0.028\pm 0.002}$    & $\pmb{0.029\pm 0.002}$      & $\pmb{0.024\pm 0.001}$ & $\pmb{0.026\pm 0.001}$ & $\pmb{0.029\pm 0.002}$ \\
                     & HMGAN$^{2}$  & $\pmb{0.033\pm 0.00}2$    & $\pmb{0.034\pm 0.001}$      & $\pmb{0.026\pm 0.001}$ & $\pmb{0.029\pm 0.001}$ & $\pmb{0.032\pm 0.002}$ \\ \hline
                     &        & dog      & frog       & horse & ship  & truck \\ \cline{3-7} 
\multirow{3}{*}{NDB $\downarrow$} & DCGAN  & $51.80\pm 3.92$    & $53.20\pm 4.27$      & $55.00\pm 2.81$ & $43.50\pm 5.00$ & $45.50\pm 5.05$ \\
                     & HMGAN$^{1}$  & $\pmb{34.00\pm 2.80}$    & $\pmb{41.60\pm 3.55}$      & $\pmb{46.50\pm 5.83}$ & $\pmb{41.50\pm 3.04}$ & $\pmb{43.20\pm 3.01}$ \\
                     & HMGAN$^{2}$  & $52.00\pm 3.02$    & $\pmb{52.10\pm 3.22}$      & $\pmb{53.90\pm 4.32}$ & $\pmb{42.80\pm 3.05}$ & $\pmb{45.30\pm 5.00}$ \\ \hline
\multirow{3}{*}{JS $\downarrow$}  & DCGAN  & $0.035\pm 0.002$    & $0.035\pm 0.002$      & $0.036\pm 0.001$ & $0.030\pm 0.002$ & $0.034\pm 0.002$ \\
                     & HMGAN$^{1}$  & $\pmb{0.024\pm 0.001}$    & $\pmb{0.029\pm 0.001}$      & $\pmb{0.032\pm 0.002}$ & $\pmb{0.027\pm 0.002}$ & $\pmb{0.028\pm 0.002}$ \\
                     & HMGAN$^{2}$  & $0.035\pm 0.002$    & $\pmb{0.034\pm 0.002}$      & $\pmb{0.031\pm 0.002}$ & $\pmb{0.028\pm 0.002}$ & $0.034\pm 0.002$ \\ \hline
\end{tabular}

\end{table*}

\begin{table}[hb]
\centering

\caption{FID and LPIPS results on the CIFAR-10 dataset.}
\label{table:2}
\begin{tabular}{p{2cm}p{2cm}p{2cm}}
\hline
Model & DCGAN & HMGAN \\ \hline
FID $\downarrow$   & $32.21\pm 0.05$ & $\pmb{28.84\pm 0.05}$ \\ \hline
LPIPS $\uparrow$ & $0.208\pm 0.002$ & $\pmb{0.209\pm 0.002}$ \\ \hline
\end{tabular}
\end{table}

\subsection{Evaluation metrics}
To evaluate the quality of the generated images, we used FID \cite{187} to measure the difference between the distribution of generated images and the distribution of real images. To compute FID, a pretrained Inception Network \cite{192} needed for extracting features of images. Lower FID indicate higher quality of the generated images.  

To evaluate diversity, we employed LPIPS \cite{186}. 
\begin{equation}\label{eq:9}
    d(\mathbf{x_1},\mathbf{x_2}) = \sum_{l}\frac{1}{H^{(l)}W^{(l)}}\sum_{h,w}\norm{w^{(l)}\odot(E_{hw}^{(l)}(\mathbf{x_1})-E_{hw}^{(l)}(\mathbf{x_2})}_{2}^{2}
\end{equation}

\begin{equation}\label{eq:10}
    diversity_{output}= \sum_{j=1}^{m}\sum_{k=1,k\neq j}^{m}d(\mathbf{x_j},\mathbf{x_k})
\end{equation}

Because deeper convolutional layers detect higher-level features \cite{190}, it is natural to measure the diversity w.r.t specific-level feature with a specific-$l$-th term in Eq. \eqref{eq:9}, 
\begin{equation}\label{eq:11}
    d^{(l)}(\mathbf{x_1},\mathbf{x_2}) = \frac{1}{H^{(l)}W^{(l)}}\sum_{h,w}\norm{w^{(l)}\odot(E_{hw}^{(l)}(\mathbf{x_1})-E_{hw}^{(l)}(\mathbf{x_2})}_{2}^{2}
\end{equation}
The $diversity^{(l)}$ is similarly computed as Eq. \eqref{eq:10} does. To statistically view the diversity of an image batch w.r.t different-level features, we visualized all $diversity^{(l)}$, as shown in figure \ref{figure:2}. Higher LPIPS means the generated images are more diverse.

To test the generated images and the real images are from the same distribution, we employed the NDB score \cite{5}. To compute NDB score, it first put all real and generated samples into bins, then, the numbers of the real images and the generated images in one bin are used to decide if those two numbers are statistically different, finally, the number of all statistically different bins defines the NDB score. The bins are the result from a K-means clustering. In other words, the K-means clustering finds k modes, so, we can not only estimate the similarity between two distributions by comparing the NDB scores but can also tell which mode has collapsed by referring the indices of statistically different bins. However, there is a trade-off between a less number of bins (less computation for the clustering) and a higher accuracy of the estimation, we presented the Jensen–Shannon divergence to validate the NDB scores, and to find a proper number of bins during the experiment. Lower NDB and JSD mean the generated images are more likely from the real distribution.

\begin{table*}[h]
\centering
\caption{Quantitative results from paired image-to-image translation task.}
\label{table:3}
\begin{tabular}{p{2cm}p{2cm}p{2cm}p{2cm}}
\hline
Datasets  & \multicolumn{3}{l}{Facades}    \\ \cline{2-4} 
          & Pix2Pix  & HMGAN$^{1}$ & HMGAN$^{2}$ \\ \cline{2-4} 
FID $\downarrow$ & $140.00\pm 2.57$    & $\pmb{90.00\pm 3.25}$    & $\pmb{138.80\pm 2.00}$    \\
NDB $\downarrow$ & $16.00\pm 0.38$    & $\pmb{12.30\pm 0.32}$    & $16.12\pm 0.59$    \\
JSD $\downarrow$ & $0.078\pm 0.003$    & $\pmb{0.028\pm 0.006}$    & $0.080\pm 0.004$    \\
LPIPS $\uparrow$ & $0.005\pm 0.001$    & $\pmb{0.192\pm 0.001}$    & $\pmb{0.007\pm 0.001}$    \\ \hline
Datasets  & \multicolumn{3}{l}{Maps}       \\ \cline{2-4} 
           & Pix2Pix  & HMGAN$^{1}$ & HMGAN$^{2}$ \\ \cline{2-4} 
FID $\downarrow$ & $165.80\pm 3.21$    & $\pmb{153.60\pm 2.50}$    & $\pmb{164.50\pm 2.39}$    \\
NDB $\downarrow$ & $47.30\pm 2,35$    & $\pmb{42.00\pm 2.52}$    & $\pmb{46.80\pm 3.52}$    \\
JSD $\downarrow$ & $0.072\pm 0.023$    & $\pmb{0.035\pm 0.003}$    & $0.076\pm 0.025$    \\
LPIPS $\uparrow$ & $0.003\pm 0.001$    & $\pmb{0.205\pm 0.001}$    & $0.003\pm 0.001$    \\ \hline
\end{tabular}
\end{table*}

\begin{figure}[h]
\centering
\includegraphics[width=80mm]{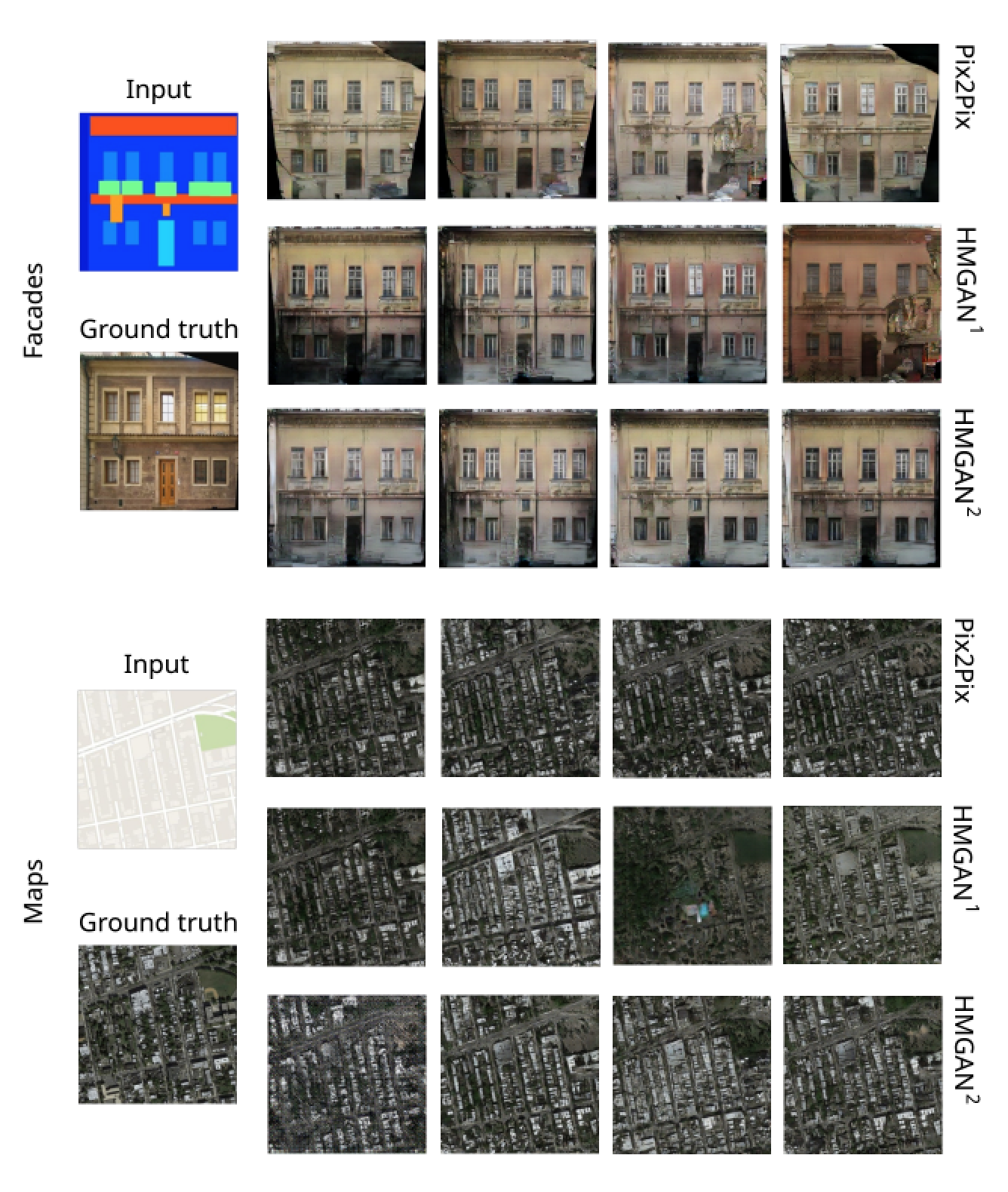}

\caption{Diversity comparison. HMGAN$^{1}$ learns more diverse results and HMGAN$^{2}$ learns less diverse results.}
\label{figure:3}
\end{figure}

\subsection{Categorical generation}
Firstly, we validated the regularization method on categorical generation task. In categorical generation, the generator takes the initial noise concatenated with class labels as input to generate images in corresponding  categories.
This task is conducted on the CIFAR-10 dataset. It has images with size $32\times 32$ in 10 categories. The NDB scores and JSD are reported in Table \ref{table:1}, and Table \ref{table:2} presents the results FID and LPIPS. The proposed method alleviates the mode collapse problem in most categories and maintains the image quality.

\begin{table*}[h]
\centering
\caption{Quantitative results from unpaired image-to-image translation task.}
\label{table:4}
\begin{tabular}{@{\extracolsep{4pt}}p{2.5cm}p{2.5cm}p{2.5cm}p{2.5cm}p{2.5cm}@{}}
\hline
Datasets  & \multicolumn{2}{l}{Summer2Winter} & \multicolumn{2}{l}{Winter2Summer} \\ \cline{2-3} \cline{4-5} 
          & DRIT        & HMGAN$^{1}$        & DRIT        & HMGAN$^{1}$        \\
FID $\downarrow$ & $55.03\pm 3.26$           & $\pmb{50.00\pm 3.23}$           & $47.00\pm 4.28$           & $\pmb{46.20\pm 3.38}$           \\
NDB $\downarrow$ & $25.50\pm 3.35$           & $\pmb{23.00\pm 0.25}$           & $29.00\pm 2.47$           & $\pmb{27.50\pm 2.55}$           \\
JSD $\downarrow$ & $0.062\pm 0.003$           & $\pmb{0.052\pm 0.003}$           & $0.050\pm 0.007$           & $\pmb{0.038\pm 0.005}$           \\
LPIPS $\uparrow$ & $0.112\pm 0.001$           & $\pmb{0.143\pm 0.001}$           & $0.112\pm 0.001$           & $\pmb{0.119\pm 0.001}$           \\ \hline
Datasets  & \multicolumn{2}{l}{Cat2Dog}       & \multicolumn{2}{l}{Dog2Cat}       \\ \cline{2-3} \cline{4-5}
          & DRIT        & HMGAN$^{1}$        & DRIT        & HMGAN$^{1}$       \\
FID $\downarrow$ & $22.50\pm 0.35$           & $\pmb{16.02\pm 0.35}$           & $59.05\pm 0.31$           & $\pmb{28.97\pm 0.54}$           \\
NDB $\downarrow$ & $39.28\pm 3.36$           & $\pmb{27.00\pm 0.50}$           & $41.32\pm 0.52$           & $\pmb{32.23\pm 0.53}$           \\
JSD $\downarrow$ & $0.125\pm 0.003$           & $\pmb{0.085\pm 0.001}$           & $0.269\pm 0.002$           & $\pmb{0.071\pm 0.001}$           \\
LPIPS $\uparrow$ & $0.250\pm 0.002$           & $\pmb{0.280\pm 0.002}$           & $0.100\pm 0.002$           & $\pmb{0.220\pm 0.003}$           \\ \hline
\end{tabular}
\end{table*}

\begin{table*}[]
\centering
\caption{Quantitative results from text-to-image generation task. HMGAN$^{3}$ refers $\lambda^{(i)}=0.5$ for $\forall{i}$.}
\label{table:5}
\begin{tabular}{@{\extracolsep{4pt}}p{2.5cm}p{2.5cm}p{2.5cm}p{2.5cm}p{2.5cm}@{}}
\hline
          
          & StackGAN++                  & HMGAN$^{1}$ & HMGAN$^{2}$              & HMGAN$^{3}$             \\
FID $\downarrow$ & $26.00\pm 4.23$                     & $\pmb{25.40\pm 2.00}$                    & $27.00\pm 1.25$                 & $\pmb{25.55\pm 1.50}$                 \\
NDB $\downarrow$ & $37.80\pm 2.44$                     & $\pmb{29.90\pm 2.55}$                    & $\pmb{37.55\pm 1.83}$                 & $\pmb{30.00\pm 3.82}$                 \\
JSD $\downarrow$ & $0.091\pm 0.005$                     & $\pmb{0.070\pm 0.005}$                    & $0.093\pm 0.005$                 & $\pmb{0.072\pm 0.003}$                 \\
LPIPS $\uparrow$ & $0.364\pm 0.005$                     & $\pmb{0.376\pm 0.005}$                    & $0.358\pm 0.005$                 & $\pmb{0.374\pm 0.002}$                 \\ \hline
\end{tabular}
\end{table*}

\begin{figure*}[h]
\centering
\includegraphics[width=160mm]{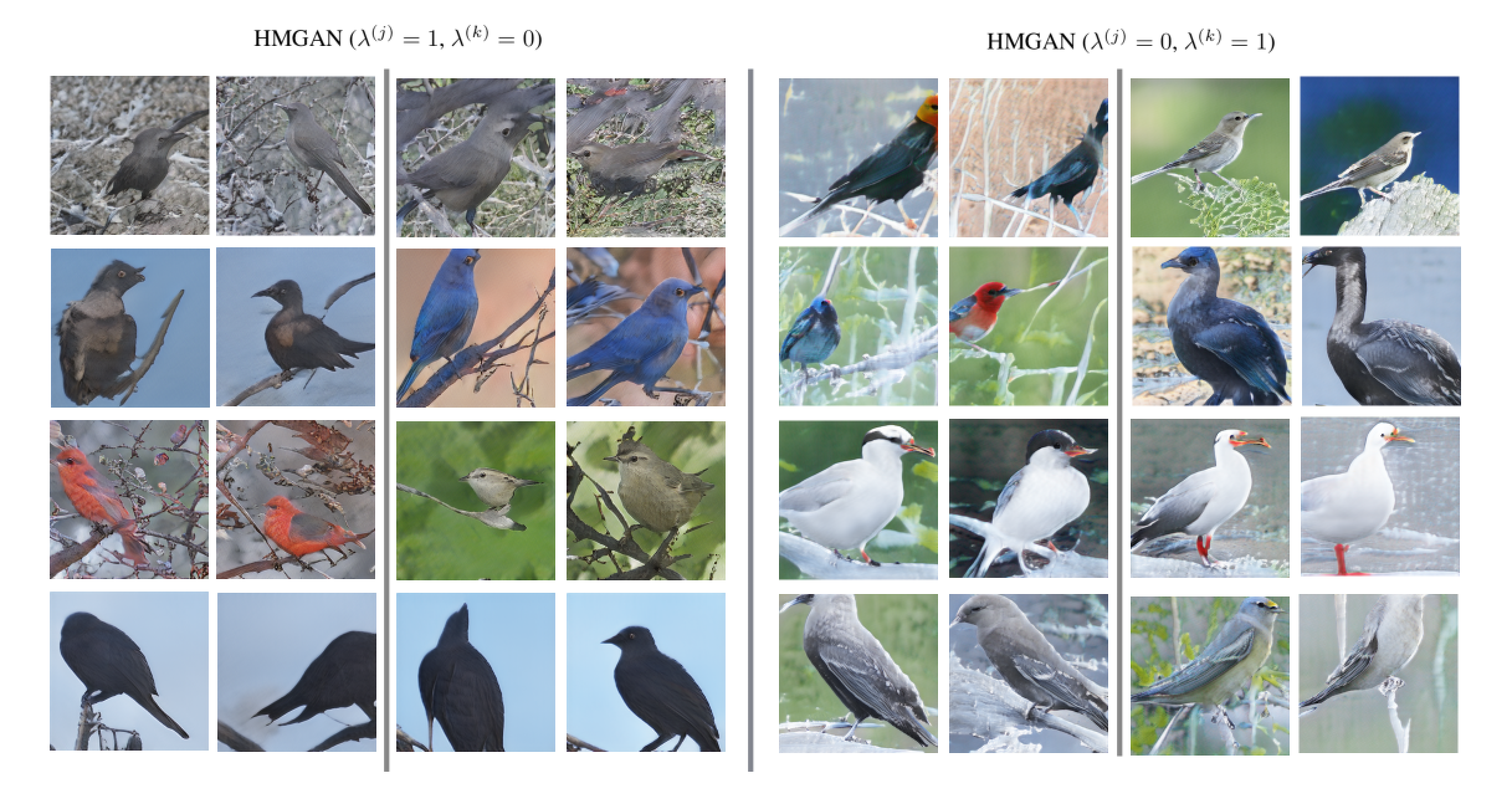}

\caption{Diversity comparison. Each pair is conditioned by the same sentence. Since StackGAN++ has 15 convolutional layers, the 6~15th layers are designed to improved the resolution of the 5th layer's output, the diversity w.r.t high-level features is controlled by 4th and 5th layer. Here, $j\in \{1,2,3,6,7,8,9,10,11,12,13,14,15\}$ and $k\in \{4,5\}$.}
\label{figure:4}
\end{figure*}

\subsection{Image-to-image translation}
\textit{Conditioned on paired images}

In this task, we integrated the proposed method into Pix2Pix.  In experiments, we kept the original hyper-parameters setting of Pix2Pix for fair comparison. Figure \ref{figure:3} and Table \ref{table:3} shows the qualitative and quantitative results, respectively. It is shown that the proposed method exceeds Pix2Pix in terms of all metrics when $\lambda^{(i)}=0$ for $\forall i$, and the output images from the proposed method have comparable diversity to the ones from Pix2Pix when $\lambda^{(i)}=1$ for $\forall i$. The low quality of generated facades images might be caused by encouraging diversity too much \cite{175}, since setting $\lambda^{(i)}=0$ for $\forall i$ regularizes the training more strictly than minimizing $\mathcal{L}_{d}$.

\noindent
\textit{Conditioned on unpaired images}

To generate images when paired images are not available, we chose DRIT  as the baseline model. It is pointed out that DRIT can generate diverse images only w.r.t low-level features, in other words, the output images share similar structures. To demonstrate the proposed method can improve the diversity w.r.t high-level features, we conducted this experiment on cat$\rightleftharpoons$dog dataset whose images has shape variations. We also compared the abilities of generating diverse image w.r.t low-level features between the proposed method and DRIT, this experiment is conducted on shape-invariant Yosemite dataset.

Table \ref{table:4} shows that the proposed method outperforms DRIT in terms of all metric in both experiments, especially on the cat$\rightleftharpoons$dog dataset. To quantitively present the difference of ability to generate diverse images w.r.t different-level features, we ploted all $diversity^{(l)}$  in figure \ref{figure:2}. Figure \ref{figure:2}. shows that our proposed method improved the diversity w.r.t high-level features, and has comparable ability to generate diverse images w.r.t low-level features. 

\begin{figure}[h]
\centering
\includegraphics[width=80mm]{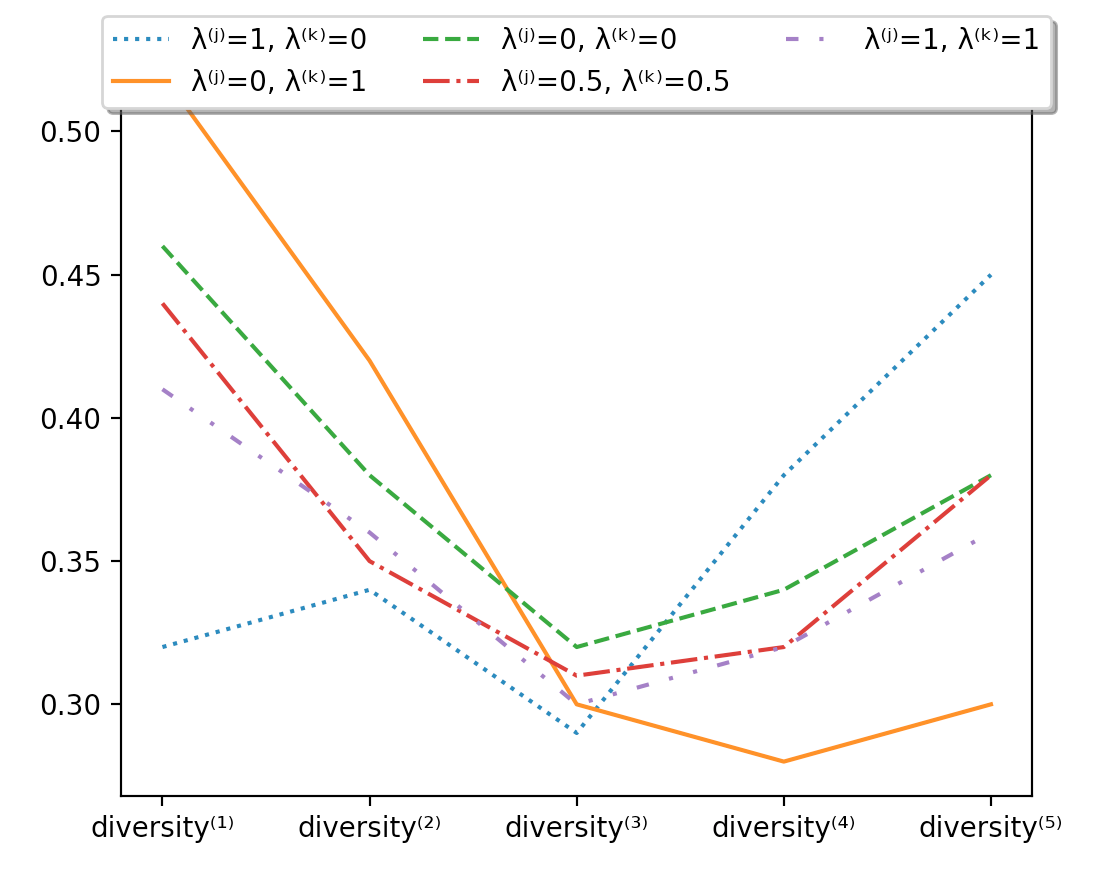}

\caption{Visualizing diversity of image batch w.r.t different-level features in text-to-image task.}
\label{figure:5}
\end{figure}

\begin{figure}[h]
\centering
\includegraphics[width=80mm]{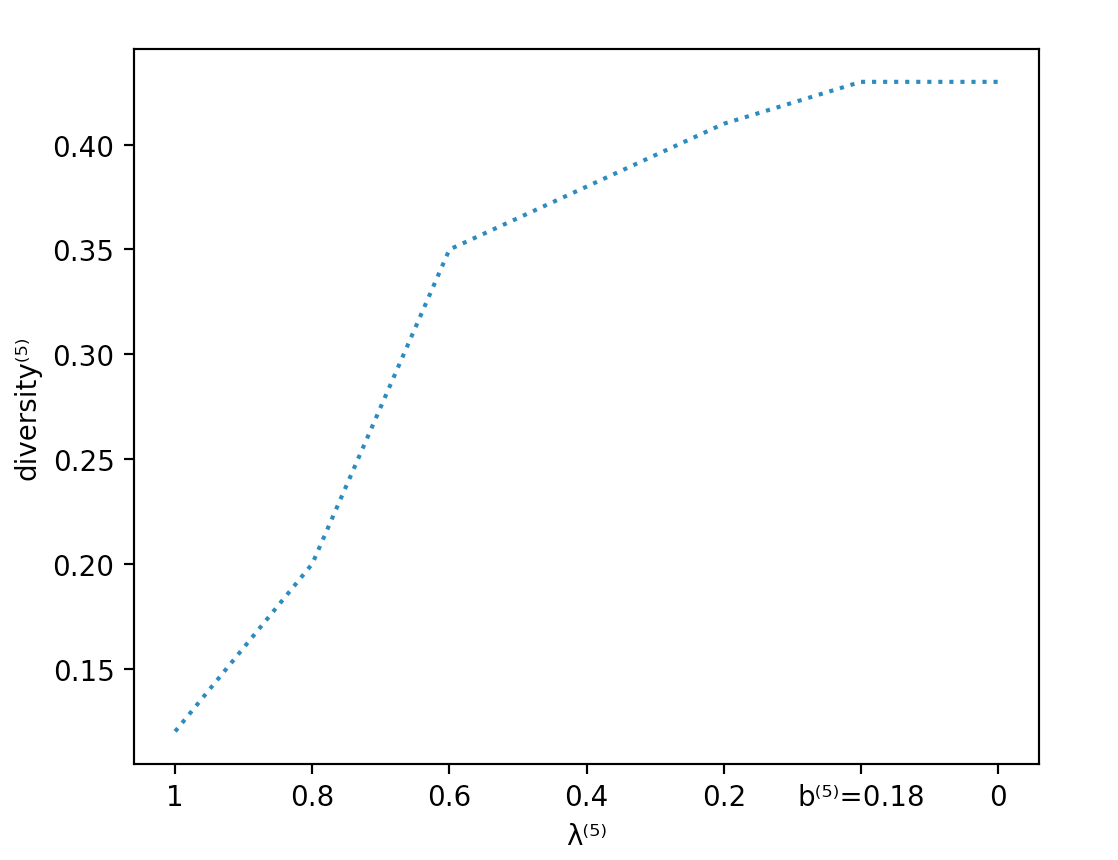}

\caption{Controlling the diversity w.r.t specific-level features by tuning one term of ERE.}
\label{figure:6}
\end{figure}

\subsection{Text-to-image generation}
StackGAN++ is proposed to generate diverse images whose contents is corresponding to given descriptive sentences. We chose it as the baseline model in this task, and the task is conducted on the CUB- 200-2011 dataset. 

Table \ref{table:5} presents quantitative comparisons between the proposed method and StackGAN++. And the qualitative results are shown in figure \ref{figure:4}, it shows that the proposed method improve the diversity without losing visual quality. 

\subsection{Controlling diversity}
To control diversity w.r.t specific level features, we chose $\lambda^{(5)}$ as the control variable in this experiment. We firstly computed the lower bound of $\lambda^{(5)}$ by choosing the minimum element in $A^{(5)}$ computed by Eq. \eqref{eq:8}. Figure \ref{figure:6} shows the results in text-to-image synthesis task, we can see $diversity^{(5)}$ is bigger with smaller $\lambda^{(5)}$, and reaches the limit when $\lambda^{(5)}<b^{(5)}$.
Figure \ref{figure:2} and figure \ref{figure:5} show that, in image translation and text-to-image generation, our method can generate outputs with different distributions of $diversity^{(l)}$, which is unachievable to the previous method. We also noticed that, in figure \ref{figure:4}, when $\lambda^{(j)}=1,\lambda^{(k)}=0$, the proposed method tends to change the observation angle to the bird, or to change the posture of the bird.

\subsection{Supplementary results}
We also conducted the three conditional image synthesis tasks using Eq. \eqref{eq:1}. The results show that our method outperforms the one using Eq. \eqref{eq:1} in tasks of categorical generation (85\% of the results are better), paired image-to-image generation (50\% of the results are better), unpaired image-to-image generation (69\% of the results are better) and text-to-image generation (100\% of the results are better).

\section{Related Work}
Unlike standard GANs only require an initial noise as input for the generator, cGANs concatenates external information (e.g. , the number of age) with the initial noise, during training, the correspondence between perceptual features (e.g., wrinkles of a face) and the additional information can be learned, as a result, an image with specific feature can be synthesized by a generator conditioned on the external information. However, it does not only inherit the mode collapse problem in standard GANs, but also worsen it when the input has a high-dimension information part \cite{175} . And it is pointed out that the noise vector is responsible for generating various images, due to its comparative low dimension, it is often ignored by the generator \cite{115}. More specifically, because the input pair has the same external information, once it is propagated through a convolutional layer, the distance between an output pair is smaller, especially when the external information has a high dimension. In this situation, two modes are prone to be collapsed into one if their initial noises are close.

    To alleviate mode collapse in cGANs, some approaches are proposed by recent works. In text-to-image tasks, \cite{74} uses a fully connected layer to sample additional noise from a Gaussian, the noise is then combined with the feature of an image as a whole conditional context to obtain more training pairs for augmentation. A different approach proposes an extra encoder network which can generate the noise vector given the generated image to help the generator construct a one-to-one mapping between the input and output, this approach was employed in image-to-image translation \cite{184}. However, the two approaches require extra time to generate augmentation pairs or to train an additional encoder, not to mention they are substantial task-specific modifications to cGANs, that is to say, they are less generalizable and charge more computational resource. Recently, \cite{115} and \cite{175} propose a regularization method (Diversity-sensitive cGANs) to amplify the diversity of the outputs, more specifically,  the regularization term encourages the generator to maximize the ratio of the distance between a noise pair to the distance between an image pair. The method needs no training overheads and can be easily extended to other frameworks of cGANs. But it ignores the diversity in hierarchical feature spaces, one of the results is that Diversity-sensitive GANs can generate the images of bird with various posture but is not able to synthesize different feather textures.

\section{Conclusion}
In this work, we applied a regularization term on the generator to address the mode collapse problem in cGANs. And we avoided  the computation cost for searching a hyperparameter growing exponentially with the number of layers of the generator. We minimized the differences between the real change of feature distance and a target change at all convolutional layers of the generator to control diversities w.r.t specific-level features. The proposed regularization term could be integrated into the existing different frameworks of cGANs. Our method is demonstrated on three image generation tasks and experimental results showed that our regularization can increase the diversity without decreasing visual quality.  As a future work, we will add more convolutional layers in the generator and validate how to control diversities more precisely. We also hope to conduct more experiments to find the dependencies of $ratio^{(i)}$.

\section{Acknowledgement}
Jinlong Li is supported by the National Key Research and Development Program of China (Grant No. 2017YFC0804001) and the National Natural Science Foundation of China (Grant No. 61573328).

\bibliographystyle{named}
\bibliography{mybib}{}

\end{document}